\documentclass{article}

\usepackage{amsthm}
\begingroup
    \makeatletter
    \@for\theoremstyle:=definition,remark,plain\do{%
        \expandafter\g@addto@macro\csname th@\theoremstyle\endcsname{%
            \addtolength\thm@preskip\parskip
            }%
        }
\endgroup

\usepackage{amsmath}
\usepackage{amssymb}
\usepackage{dsfont}
\usepackage{hyperref}
\usepackage{enumitem}
\usepackage{graphicx}
\usepackage{tikz}
\usetikzlibrary{positioning}
\usepackage[title]{appendix}

\newtheorem{AssumptionGeneral}{Assumption}
\newtheorem{theorem}{Theorem}[section]
\newtheorem{Def}[theorem]{Definition}
\newtheorem{Not}[theorem]{Notation}
\newtheorem{Anm}[theorem]{Remark}

\newtheorem{Lemma}[theorem]{Lemma}
\newtheorem{Kor}[theorem]{Corollary}
\newtheorem{Satz}[theorem]{Theorem}

\usepackage{arxiv}

\newcommand{\closedBall}[2]{\mathcal{B}[#1, #2]}
\newcommand{\coveringNumber}[2]{\mathcal{N}_c(#1, #2)}
\newcommand{\sprod}[2]{\langle #1, #2 \rangle}
\newcommand{\divergence}[3]{D_{#1}(#2 \ \Vert \ #3)}
\newcommand{\rvec}[1]{\mathbf{\mathfrak{#1}}}

\newcommand{\pspace}{\bigl(\Omega, \mathcal{F}, \mathbb{P} \bigr)}

\DeclareMathOperator*{\argmin}{arg\,min}

\usepackage[round]{natbib}

\title{PAC-Bayesian Learning of Optimization Algorithms}

\author{
 Michael Sucker \\
  Department of Mathematics \\
  University of T\"ubingen\\
  \texttt{michael.sucker@math.uni-tuebingen.de} \\
   \And
 Peter Ochs \\
  Department of Mathematics\\
  University of T\"ubingen\\
  \texttt{ochs@math.uni-tuebingen.de} \\
}

\begin{document}
\maketitle

\begin{abstract}
We apply the PAC-Bayes theory to the setting of learning-to-optimize. To the best of our knowledge, we present the first framework to learn optimization algorithms with provable generalization guarantees (PAC-bounds) and explicit trade-off between a high probability of convergence and a high convergence speed. Even in the limit case, where convergence is guaranteed, our learned optimization algorithms provably outperform related algorithms based on a (deterministic) worst-case analysis. Our results rely on PAC-Bayes bounds for general, unbounded loss-functions based on exponential families. By generalizing existing ideas, we reformulate the learning procedure into a one-dimensional minimization problem and study the possibility to find a global minimum, which enables the algorithmic realization of the learning procedure. As a proof-of-concept, we learn hyperparameters of standard optimization algorithms to empirically underline our theory.
\end{abstract}


\section{Introduction}
Let $\ell(\cdot,\theta)$ be an instance of a class of functions $(\ell(\cdot, \theta))_{\theta \in \Theta}$.
We consider the minimization problem: 
\begin{equation}
    \min_{x \in \mathbb{R}^n} \ \ell(x, \theta) \,.
\end{equation}
Our goal is the construction of an algorithm $\mathcal{A}(\alpha, \theta)$, depending on some hyperparameters $\alpha$, that is provably the best (on average) for the given class of problems. We contrast the majority of approaches in continuous optimization in two ways:
\begin{enumerate}
    \item[i)] Classical optimization theory studies the worst-case behaviour, which guarantees the same convergence for all problems that arise:
    $$
    \alpha^* \in \argmin_{\alpha \in \mathcal{H}} \ \sup_{\theta \in \Theta} \ell(\mathcal{A}(\alpha, \theta), \theta) \,.
    $$ 
    Thereby, this is often accompanied by rough estimates and ignores that some problems are more likely to occur than others. On the other hand, by using the additional information that $\theta$ is a realization of some random variable $\rvec{S}$, we seek for the average case in form of the mean function, usually called the risk: 
    $$
    \alpha^* \in \argmin_{\alpha \in \mathcal{H}} \mathbb{E}_\rvec{S}[\ell(\mathcal{A}(\alpha, \rvec{S}), \rvec{S})] \,.
    $$ 
    From an optimization perspective, this is a distinct approach leading to performance guarantees in expectation or with high probability over the draw of new problem instances. This allows us to exploit features of the considered class of problems beyond analytical accessible quantities such as the Lipschitz constant (of the gradient) or the strong convexity modulus, which are usually pessimistic and hard to compute.
    \item[ii)] Instead of analytically constructing an algorithm driven by intricate worst-case estimates, we train our algorithm (by learning) to be the best one on some samples $\{\ell(\cdot, \theta_i)\}_{i=1}^N$ and prove that the performance generalizes, in a suitable sense (PAC-Bayes), to the random function $\ell(\cdot, \rvec{S})$. This type of problem, i.e. minimizing the expected loss, is naturally found in the whole area of machine learning and cannot be solved directly, since the mean function is generally unknown. Consequently, one typically solves an approximate problem like empirical risk minimization in the hope that the solution found there will transfer: 
    $$
        \alpha^* \in \argmin_{\alpha \in \mathcal{H}} \frac{1}{N} \sum_{i=1}^N \ell(\mathcal{A}(\alpha, \theta_i), \theta_i) \,.
    $$
    However, through this, one is left with the problem of generalization, which is one of the key problems in machine learning in general. Therefore, one of the main concerns of learning-to-optimize are generalization bounds. A famous framework to provide such bounds is the PAC-Bayes framework, which allows for giving high-probability bounds on the true risk relative to the empirical risk.
\end{enumerate}
In this paper, we apply the PAC-Bayesian theory to the setting of learning-to-optimize. In doing so, we provide PAC-Bayesian generalization bounds for a general optimization algorithm on a general, unbounded loss function and we show how one can trade-off convergence guarantees for convergence speed. As a proof of concept, we illustrate our approach by learning, for example, the step-size $\tau$ and the inertial parameter $\beta$, i.e., $\alpha = (\tau, \beta)$, of a fixed number of iterations of the Heavy-ball update scheme given by:
\begin{equation}\label{Eq_UpdateHB}
    \begin{split}
        x^{(k+1)} 
        = \mathcal{HB}\Bigl(x^{(k)}, x^{(k-1)}, \alpha, \theta \Bigr) 
        := x^{(k)} - \tau \nabla \ell(x^{(k)}, \theta) + \beta (x^{(k)} - x^{(k-1)}),
    \end{split}
\end{equation}
which generalizes Gradient Descent for $\beta = 0$.

\subsection{Our Contributions}

\begin{itemize}
    \item We provide a general PAC-Bayes theorem for general, unbounded loss functions based on exponential families. In this framework, the role of the reference distribution (called the prior), the data dependence of the learned distribution (called the posterior) and the divergence term arise directly and naturally from the definition. Furthermore, this abstract approach allows for a unified implementation of the learning framework.
    \item We provide a principled way of excluding the case of the learnt algorithm's divergence from the considerations, which in turn allows us to apply our PAC theorem under a modified objective. Based on this, we give a theoretically grounded way of ensuring a given (user-specified) convergence probability during learning. Taken together, this allows us to trade-off convergence speed and the probability of convergence. To the best of our knowledge, both approaches are completely new and could also be very interesting for other learning approaches.
    \item We apply our PAC-Bayesian framework to the problem of learning-to-optimize and learn optimization algorithms by minimizing the PAC-Bayesian upper bound.
\end{itemize}

\section{Related Work}
The literature on both learning-to-optimize and the PAC-Bayes learning approach is vast. Hence, in the discussion of learning-to-optimize, we will mainly focus on approaches that provide certain theoretical guarantees. Especially, this excludes most model-free approaches, which replace the whole update step with a learnable mapping such as a neural network. \cite{CCCHLWY2021} provide a good overview of the variety of approaches in learning-to-optimize. Good introductory references for the PAC-Bayes approach are given by \cite{G2019} and \cite{A2021}. 

\paragraph{Learning-to-Optimize with Guarantees.} 
\cite{CCCHLWY2021} point out that learned optimization methods may lack theoretical guarantees for the sake of convergence speed. That said, there are applications where convergence guarantee is of highest priority. To underline this problem, \cite{MMC2019} provide an example where a purely learning-based approach fails to reconstruct the crucial details in medical image reconstruction. Also, they prove convergence of their method by restricting the output to descent directions, for which mathematical guarantees exist. The basic idea is to trace the learned object back to, or constrain it to, a mathematical object with convergence guarantees. Similarly, \cite{SVWBDSB2016} provide sufficient conditions under which the learned mapping is a proximal mapping. Related schemes under different assumptions and guarantees are given by \cite{CWE2016}, \cite{TBF2017}, \cite{TG2018}, \cite{BCSB2018}, \cite{RLWCWY2019}, \cite{SWK2019}, \cite{TRPW2021} and \cite{CEM2021}. A major advantage of these methods is the fact that the number of iterations is not restricted a priori. However, a major drawback is their restriction to specific algorithms and problems. Another approach, which limits the number of iterations, yet in principle can be applied to every iterative optimization algorithm, is unrolling, pioneered by \cite{GL2010}. \cite{XWGWW2016} study the IHT algorithm and show that it is, under some assumptions, able to achieve a linear convergence rate. Likewise, \cite{CLWY2018} establish a linear convergence rate for the unrolled ISTA. However, a difficulty in the theoretical analysis of unrolled algorithms is actually the notion of convergence itself, such that one rather has to consider the generalization performance. Only few works have addressed this: Either directly by means of Rademacher complexity \citep{CZRS2020}, or indirectly in form of a stability analysis \citep{KEKP2020}, as algorithmic stability is linked to generalization and learnability \citep{BE2000, BE2002, SSSS2010}. \textit{We consider the model-based approach of unrolling a general iterative optimization algorithm and provide generalization guarantees in form of PAC-bounds.}

\paragraph{PAC-Bounds through Change-of-Measure.}
The PAC-Bayesian framework allows us to give high probability bounds on the risk, either as an oracle bound or as an empirical bound. The key ingredients are so-called change-of-measure inequalities. The choice of such an inequality strongly influences the corresponding bound. The one used most often is based on a variational representation of the Kullback--Leibler divergence due to \cite{DV1975}, employed, for example, by \cite{C2004, C2007}. Yet, not all bounds are based on a variational representation, i.e., holding uniformly over all posterior distributions  \citep{RKSS2020}. However, most bounds involve the Kullback--Leibler divergence as a measure of proximity, e.g. those by \cite{Mc2003_1, Mc2003_2}, \cite{S2002}, \cite{LS2002}, or the general PAC-Bayes bound of \cite{GLL2009}. More recently, other divergences have been used: \cite{HJ2014} prove an inequality for the $\chi^2$-divergence, which is also used by \cite{L2017}. \cite{BGLR2016} and \cite{AG2018} use the Renyi-divergence ($\alpha$-divergence). \cite{OH2021} propose several PAC-bounds based on the general notion of f-divergences, which includes the Kullback--Leibler-, $\alpha$- and $\chi^2$-divergences. \textit{We develop a general PAC theorem based on exponential families. In this general approach, the role of prior, posterior, divergence and data dependence will be given naturally. Moreover, this approach allows us to implement a general learning framework that can be applied to a wide variety of algorithms.}

\paragraph{Boundedness of the Loss Function.}
A major drawback of many of the existing PAC-Bayes bounds is the assumption of a bounded loss-function. However, this assumption is mainly there to apply some exponential moment inequality like the Hoeffding- or Bernstein-inequality \citep{RKSS2020, A2021}.  Several ways have been developed to solve this problem: \cite{GLL2009} explicitly include the exponential moment in the bound, \cite{ARC2016} use so-called Hoeffding- and Bernstein-assumptions, \cite{C2004} restricts to the sub-Gaussian or sub-Gamma case. Another possibility, of which we make use of here, is to ensure the generalization or exponential moment bounds by properties of the algorithm in question. \cite{L2017} uses algorithmic stability to provide PAC-Bayes bounds for SGD. \textit{We consider suitable properties of optimization algorithms aside from algorithmic stability to ensure the exponential moment bounds. To the best of our knowledge, this has not been done before.}

\paragraph{Minimization of the PAC-Bound.}
The PAC-bound is a relative bound and relates the risk to other terms such as the empirical risk. Yet, it does not directly say anything about the actual numbers. Thus, one aims to minimize the bound: \cite{LC2001} compute non-vacuous numerical generalization bounds through a combination of PAC-bounds with a sensitivity analysis. \cite{DR2017} extend this by minimizing the PAC-bound directly. \cite{PRSS2021} also consider learning by minimizing the PAC-Bayes bound and provide very tight generalization bounds. \cite{TIWS2017} are able to solve the  minimization problem resulting from their PAC-bound by alternating minimization. Further, they provide sufficient conditions under which the resulting minimization problem is quasi-convex. \textit{We also follow this approach and consider learning as minimization of the PAC bound, however, applied to the context of learning-to-optimize.}

\paragraph{Choice of the Prior.}
A common difficulty in learning with PAC-Bayes bounds is the choice of the prior distribution, as it heavily influences the performance of the learned models and the generalization bound \citep{C2004, DHGAR2021, PRSS2021}. In part, this is due to the fact that the divergence term can dominate the bound, keeping the posterior close to the prior. This leads to the idea to choose a data- or distribution-dependent prior \citep{S2002, PASS2012, LLS2013, DR2018, PRSS2021}. \textit{As we also found the choice of the prior distribution to be crucial for the performance of our learned algorithms, we use a data-dependent prior. Further, we point out how the prior is essential in preserving necessary properties during learning. It is key to control the trade-off between convergence guarantee and convergence speed.}

\paragraph{More Generalization Bounds.}
There are many areas of research that study generalization bounds and have not been discussed here. Importantly, the vast field of stochastic optimization (SO) provides generalization bounds for specific algorithms. The main differences to our setting are the learning approach and the assumptions made:
\begin{itemize}[leftmargin=*]
    \item Learning approach: In most of the cases, the concrete algorithms studied in SO generate a single point by either minimizing the (regularized) empirical risk functional over a possibly large dataset, or by repeatedly updating the point estimate based on a newly drawn (small) batch of samples. Then, one studies the properties of this point in terms of the stationarity measure of the true risk functional (see e.g. \cite{BCN2018, DavisDrusvyatskiy2022, BHS2022}).
    \item Assumptions: Since the setting in SO is more explicit, more assumptions have to be made. Typical assumptions in SO are (weak) convexity \citep{SSSS2009_SCO, DavisDrusvyatskiy2019}, bounded gradients \citep{DefossezBottouBachUsunier2022}, bounded noise \citep{DavisDrusvyatskiy2022}, or at least smoothness \citep{KavisLevyCevher2022}, just to name a few.
\end{itemize}
\emph{We provide a principled way to learn a distribution over hyperparameters of an abstract algorithm under weak assumptions. Further, the methodology is independent of a concrete implementation and independent of the concrete choice of hyperparameters. Furthermore, we go explicitly beyond analytically tractable quantities.}

\section{Preliminaries and Notation}
If not further specified, we will endow every topological space $X$ with the corresponding Borel-$\sigma$-algebra $\mathcal{B}(X)$. If we consider a product space $X \times Y$ of two measurable spaces $(X, \mathcal{A})$ and $(Y, \mathcal{B})$, we endow it with the product-$\sigma$-algebra $\mathcal{A} \otimes \mathcal{B}$. We use the Fraktur-font to denote random variables. Let $\pspace$ be a probability space, $\Theta$ be a Polish space and 
$$
\rvec{S}: \pspace \longrightarrow \Theta
$$ 
be a random variable. Its distribution is denoted by $\mathbb{P}_\rvec{S}$, following the general notation $\mathbb{P}_\rvec{X}$ to denote the distribution of a random variable $\rvec{X}$. Integration w.r.t. $\mathbb{P}_\rvec{X}$ is denoted by $\mathbb{E}_\rvec{X}[g] := \mathbb{E}_\rvec{X}[g(\rvec{X})] := \int g(x) \ \mathbb{P}_\rvec{X}(dx)$. Finally, $\mathds{1}_A$ denotes the indicator function of a set A, which is one for $x \in A$ and zero else, and $\log$ denotes the natural logarithm.

\begin{Def}
    Let $N \in \mathbb{N}$. Further, let $\pspace$ be a probability space, $\rvec{S}_i: \pspace \longrightarrow \Theta$, $i = 1,...,N$, be random variables. A measurable function 
    \begin{align*}
        \rvec{D}_N: \pspace \longrightarrow \Bigl( \prod_{i=1}^N \Theta,  \bigotimes_{i=1}^N \mathcal{B}(\Theta) \Bigr), \  
        \omega \mapsto \prod_{i=1}^N \rvec{S}_i(\omega)
    \end{align*}
    is called a \emph{dataset}. If the induced distribution $\mathbb{P}_{\rvec{D}_N}$ factorizes into the product of the marginals, i.e., if it satisfies
    $\mathbb{P}_{\rvec{D}_N} = \bigotimes_{i=1}^N \mathbb{P}_{\rvec{S}_i}$,
    it is called \emph{independent} and if, additionally, it satisfies
    $\mathbb{P}_{\rvec{D}_N} = \bigotimes_{i=1}^N \mathbb{P}_{\rvec{S}}$, it is called \emph{i.i.d.}
\end{Def}

\begin{Not}
    The space $( \prod_{i=1}^N \Theta, \ \bigotimes_{i=1}^N \mathcal{B}(\Theta) )$ will be denoted by $(\mathcal{D}_N, \mathcal{B}(\mathcal{D}_N))$. Since $\bigotimes_{i=1}^N \mathcal{B}(\Theta)$ is indeed the Borel-$\sigma$-algebra of $\mathcal{D}_N$, it will not be mentioned anymore.
\end{Not}

In the PAC-Bayesian framework, generalization bounds typically involve a so-called posterior distribution, which in turn is referred to as a data-dependent distribution. Often, this term is left unspecified. However, as also pointed out by \cite{RKSS2020}, this is an instance of a Markov kernel. Another commonly used name are regular conditional probabilities, following the definition of a regular conditional distribution \citep{C2004, A2008}.

\begin{Def}\label{Def_Data_Dependent_Distribution}
    Let $\rvec{D}_N: \pspace \longrightarrow \mathcal{D}_N$ be a dataset and $\mathcal{H}$ a Polish space. A Markov kernel from $\mathcal{D}_N$ to $\mathcal{H}$ is called a \emph{data-dependent distribution}.
\end{Def}

\begin{Anm}
    The assumption of a Polish space is not very restrictive (for practical considerations) and sufficient to ensure the existence of such Markov kernels. Both definitions can be found in the supplementary material \ref{Appendix_Definitions}.
\end{Anm}

The following theory will be based on exponential families, which are a special class of probability distributions with a specific, mathematically convenient form. 

\begin{Def}
    Let $\Lambda \subset \mathbb{R}^k$. A family of probability measures $(\mathbb{Q}_\lambda)_{\lambda \in \Lambda}$ on a measurable space $(\mathcal{H}, \mathcal{B}(\mathcal{H}))$ is called an \emph{exponential family}, if there is a dominating probability measure $\mathbb{P}_\rvec{H}$, measurable functions $\eta_1, ..., \eta_k: \Lambda \longrightarrow \mathbb{R}$, a measurable function $A : \Lambda \longrightarrow \mathbb{R}_{>0}$, measurable functions $T_1,...,T_k: \mathcal{H} \longrightarrow \mathbb{R}$ and $h: \mathcal{H} \longrightarrow \mathbb{R}_{> 0}$, such that every $\mathbb{Q}_\lambda$ has a $\mathbb{P}_{\rvec{H}}$-density of the form:
    $$
        \frac{d\mathbb{Q}_\lambda}{d \mathbb{P}_\rvec{H}}(\alpha) = h(\alpha) A(\lambda) \exp\bigl(\sprod{\eta(\lambda)}{T(\alpha)}\bigr), \quad \mathbb{P}_\rvec{H}-a.s.
    $$
    where $\eta := (\eta_1,...,\eta_k)$ and $T := (T_1,...,T_k)$.
\end{Def}

In the PAC-Bayesian setting, the dominating measure $\mathbb{P}_\rvec{H}$ is usually referred to as the prior and every distribution $\mathbb{Q} \ll \mathbb{P}_\rvec{H}$ is referred to as a posterior. Note that this deviates from the standard definitions of prior and posterior in Bayesian statistics, which are linked through the likelihood.
We use a similar notation as in \cite{BN2014} and denote
\begin{equation} \label{Eq_DefKappa}
    \begin{split}
        &c(\lambda) := \int_{\mathcal{H}} h(\alpha) \exp(\sprod{\eta(\lambda)}{T(\alpha)}) \ \mathbb{P}_\rvec{H}(d\alpha) \\
        &\kappa(\lambda) := \log(c(\lambda)),
    \end{split}
\end{equation}
or short, $\kappa = \log(c)$. It holds that $A(\lambda) = c(\lambda)^{-1}$.

\begin{Anm}\label{Rem_KappaConvex}
    In the case $h = 1$ and $\eta(\lambda) = \lambda$, $c$ is the Laplace transform (moment generating function) of the push-forward measure $\mathbb{P}_\rvec{H} \circ T^{-1}$ and $\kappa$ the corresponding log-Laplace transform (cumulant-generating function). Further, if $\eta(\lambda)$ actually describes a lower-dimensional manifold or curve in $\mathbb{R}^k$, $(\mathbb{Q}_\lambda)_{\lambda \in \Lambda}$ is sometimes also called a curved exponential family \citep{E1975}.
\end{Anm}

\begin{Anm}
    In the following we will consider data-dependent exponential families, i.e., the sufficient statistic $T$ additionally depends on a dataset $\rvec{D}_N$. Hence, also $c$ and $\kappa$ do depend on $\rvec{D}_N$. Thus, we will assume that $T: \mathcal{H} \times \mathcal{D}_N \longrightarrow \mathbb{R}$ is measurable. In this case, $\mathbb{Q}_\lambda$ is indeed a data-dependent distribution.
\end{Anm}

\begin{Not}
    For notational simplicity, we will omit the dependence of $\mathbb{Q}_\lambda$, $T$, $c$ and $\kappa$ on the dataset $\rvec{D}_N$.
\end{Not}

For the rest of the paper, we assume that we are given an exponential family $(\mathbb{Q}_\lambda)_{\lambda \in \Lambda}$ w.r.t. $\mathbb{P}_\rvec{H}$ of the form:
\begin{equation}\label{EQ_ExpFamily}
    \frac{d\mathbb{Q}_\lambda}{d \mathbb{P}_\rvec{H}}(\alpha) = \frac{h(\alpha)}{c(\lambda)} \exp\Bigl(\sprod{\eta(\lambda)}{T(\alpha)}\Bigr)\,.
\end{equation}

Finally, since the loss-function is neither assumed to be bounded nor to satisfy any self-bounding or bounded-difference property, the following result will be needed. It states that non-negative random variables with finite second moment satisfy a one-sided sub-Gaussian inequality. It can be found as Exercise 2.9 on page 47 in the book by \cite{BLM2013}.

\begin{Lemma}\label{LemmaSubGaussianLowerTail}
    Let $\rvec{X}$ be a non-negative random variable with finite second moment. Then, for every $\lambda > 0$ it holds:
    $$
        \mathbb{E}\Bigl[ \exp \Bigl( -\lambda(\rvec{X} - \mathbb{E}[\rvec{X}]) \Bigr) \Bigr] \le 
        \exp \Bigl( \frac{\lambda^2}{2} \mathbb{E}[\rvec{X}^2] \Bigr) \,.
    $$
\end{Lemma}

\section{Problem Setup}
As described in the introduction, we aim to solve the following minimization problem with a random objective function $\ell$ under Assumption~\ref{AssumptionLossNonnegative}:
$$
    \min_{x \in \mathbb{R}^n} \ell(x, \rvec{S})\,.
$$
\begin{AssumptionGeneral}\label{AssumptionLossNonnegative}
    $\Theta$ is a Polish space, $\rvec{S}:\pspace \longrightarrow \Theta$ is a random variable, and $\ell: \mathbb{R}^n \times \Theta \longrightarrow \mathbb{R}$ is measurable and non-negative.
\end{AssumptionGeneral}

\begin{Anm}
    The non-negativity assumption is not restrictive, as any lower-bounded function $f$ can be rescaled according to  $\ell(x, \theta) := f(x, \theta) - \inf_{x \in \mathbb{R}^n} f(x,\theta)$.
\end{Anm}
To actually solve this problem for a concrete realization $\theta$, we apply an optimization algorithm $\mathcal{A}$ to $\ell$. For this, we will consider a similar setting as in \cite{L2017}, i.e., randomized algorithms are considered as deterministic algorithms with randomized hyperparameters. 
\begin{Def}
    Let $\mathcal{H}$ be a Polish space. A measurable function 
    $$
        \mathcal{A}: \mathcal{H} \times \mathbb{R}^n \times \Theta \longrightarrow \mathbb{R}^n
    $$
    is called a \emph{parametric algorithm} and $\mathcal{H}$ is called the \emph{hyperparameter space} of $\mathcal{A}$. A random variable
    $$
        \rvec{H}: \pspace \longrightarrow \mathcal{H}
    $$
    is called a \emph{hyperparameter} of $\mathcal{A}$.
\end{Def}

\begin{Anm}
    $\mathcal{H}$ corresponds to the hyperparameters of the algorithm, $\mathbb{R}^n$ to the initialization and $\Theta$ to the parameters of the problem instance.
\end{Anm}

Learning now refers to learning a distribution $\mathbb{Q}$ on $\mathcal{H}$. For this, one needs a reference distribution:
\begin{AssumptionGeneral}
    $\mathcal{A}$ is a parametric optimization algorithm with hyperparameter space $\mathcal{H}$. The prior $\mathbb{P}_{\rvec{H}}$ is induced by hyperparameters $\rvec{H}: \pspace \longrightarrow \mathcal{H}$ that are independent of the dataset $\rvec{D}_N$ and $\rvec{S}$. The initialization $x^{(0)} \in \mathbb{R}^n$ is given and fixed.
\end{AssumptionGeneral}

The initialization and the probability space $\pspace$ will not be mentioned anymore. 
We define the risk of a randomized parametric algorithm in the usual way:

\begin{Def}
    Let $N \in \mathbb{N}$ and let $\rvec{D}_N = (\rvec{S}_1,...,\rvec{S}_N)$ be a data set. Further, let $\mathcal{A}$ be a parametric algorithm with hyperparameter space $\mathcal{H}$. Furthermore, let $\rvec{S} \sim \mathbb{P}_\rvec{S}$ be independent of $\rvec{D}_N$. Finally, let $\ell: \mathbb{R}^n \times \Theta \longrightarrow \mathbb{R}_{\ge 0}$ satisfy Assumption~\ref{AssumptionLossNonnegative}. The risk of $\mathcal{A}$ is defined as the measurable function:
    \begin{align*}
        \mathcal{R}: \mathcal{H} \longrightarrow \mathbb{R}_{\ge 0}, \
        \alpha \mapsto \mathbb{E}_\rvec{S}[ 
        \ell \bigl(\mathcal{A}(\alpha, \rvec{S}), \rvec{S}\bigr) 
        ] \,.
    \end{align*}
    Similarly, the empirical risk of $\mathcal{A}$ on $\rvec{D}_N$ is defined as the measurable map $\hat{\mathcal{R}}: \mathcal{H} \times \mathcal{D}_N \longrightarrow \mathbb{R}_{\ge 0}$ with:
    \begin{align*}
        \hat{\mathcal{R}}(\alpha, \rvec{D}_N) =
        \frac{1}{N} \sum_{i=1}^N \ell \bigl(\mathcal{A}(\alpha, \rvec{S}_i), \rvec{S}_i \bigr) \,.
    \end{align*}
\end{Def}

\begin{Not}
    We also use $\ell(\alpha, \theta) := \ell (\mathcal{A}(\alpha, \theta), \theta)$.
\end{Not}

\section{General PAC-Bayesian Theorem}\label{Sec_PAC}
In this section we derive a general PAC-Bayes bound, which will be used to bound the generalization risk of the learned parametric optimization algorithm $\mathcal{A}$.
As stated above, PAC-Bayesian theorems are usually based on a change-of-measure (in-)equality. The following lemma is a form of the Donsker--Varadhan variational formulation. Though it is not new, we state it nevertheless for the sake of completeness. The proof is especially easy in this case and is given in the supplementary material \ref{Proof_Lem_DonskerVaradhanForExponentialFamilies}. 

\begin{Lemma}\label{Lem_DonskerVaradhanForExponentialFamilies}
    Let $(\mathbb{Q}_\lambda)_{\lambda \in \Lambda}$ be an exponential family on $\mathcal{H}$ w.r.t. $\mathbb{P}_{\rvec{H}}$ of the form (\ref{EQ_ExpFamily}) and $\kappa$ as in (\ref{Eq_DefKappa}). Then it holds:
    $$
        \kappa(\lambda) = \sup_{\mathbb{Q} \ll \mathbb{P}_\rvec{H}} \mathbb{E}_\mathbb{Q}\bigl[ \sprod{\eta(\lambda)}{T} + \log(h) \bigr] - \divergence{KL}{\mathbb{Q}}{\mathbb{P}_\rvec{H}} \,.
    $$
    Furthermore, the supremum is attained at $\mathbb{Q}_\lambda$.
\end{Lemma}

This change-of-measure equality allows to directly give the PAC-Bayesian theorem in its general form. Basically, one uses Markov's inequality to give a probabilistic bound on $\kappa(\lambda)$. The restriction to a finite set is made such that the bound also holds uniformly in $\lambda \in \Lambda$ by a union-bound. This idea appeared previously \citep{LC2001, C2007, A2021}.

\begin{Satz}\label{PAPER_PACBayesExponentialFamiliy}
    Let $(\mathbb{Q}_\lambda)_{\lambda \in \Lambda}$ be an exponential family on $\mathcal{H}$ of the form (\ref{EQ_ExpFamily}). Further, let $\Lambda$ be a finite set with cardinality $\vert \Lambda \vert$ and let $\mathbb{E}_{\rvec{D}_N}[c(\lambda)] \le 1$ for all $\lambda \in \Lambda$. Then, for $\epsilon > 0$, it holds that:
    \begin{align*}
        \mathbb{P}_{\rvec{D}_N} \Bigl\{ \forall \lambda \in \Lambda, \ \forall \mathbb{Q} \ll \mathbb{P}_\rvec{H}  :  
        \mathbb{E}_{\mathbb{Q}}\bigl[ \sprod{\eta(\lambda)}{T} + \log(h) \bigr] 
        \le 
        \divergence{KL}{\mathbb{Q}}{\mathbb{P}_\rvec{H}} + \log\Bigl(\frac{\vert \Lambda \vert}{\epsilon}\Bigr)
        \Bigr\} \ge 1-\epsilon \,.
    \end{align*}
\end{Satz}

The proof of Theorem \ref{PAPER_PACBayesExponentialFamiliy} is given in the supplementary material \ref{Proof_PAPER_PACBayesExponentialFamiliy}.

\begin{Anm}
    The restriction to a finite set gets problematic, if the term $\log(\vert \Lambda \vert)$ influences the bound strongly. In our application the loss is usually much larger than $\log(\vert \Lambda \vert)$, such that this is not the case even for large $\vert \Lambda \vert$. 
\end{Anm}

\begin{Anm}
    By a chaining argument, the finiteness assumption on $\Lambda$ can be relaxed to assuming that $\Lambda$ is totally bounded (e.g. compact) and that the growth of $\kappa$ can be controlled on balls of radius $\delta$. For more details, we refer to the supplementary material \ref{proof_LambdaFiniteNotNecessary}, as this was only found after the rebuttal phase. Also, note that the experiments in Section \ref{Sec_Experiments} were carried out with the setting in Theorem \ref{PAPER_PACBayesExponentialFamiliy}.
\end{Anm}

For the rest of the paper, we will have $h \equiv 1$. Corollary~\ref{PAPER_KorGeneralizationBoundExpFamily} shows how to transform this general result into a high-probability bound on the risk. It follows directly by using the properties of the Euclidean scalar product.

\begin{Kor}\label{PAPER_KorGeneralizationBoundExpFamily}
    Let the assumptions of Theorem~\ref{PAPER_PACBayesExponentialFamiliy} hold. Furthermore, assume that there are $T': \mathcal{H} \times \mathcal{D}_N \longrightarrow \mathbb{R}^{k-1}$, $\eta': \Lambda \longrightarrow \mathbb{R}^{k-1}$ and $\eta_1: \Lambda \longrightarrow \mathbb{R}_{> 0}$, such that $\eta$ and $T$ are given by:
    \begin{align*}
        \eta(\lambda) = \bigl(\eta_1(\lambda), \ \eta'(\lambda) \bigr)\, , \ \ \ 
        T(\alpha, \rvec{D}_N) = \bigl(\mathcal{R}(\alpha) - \hat{\mathcal{R}}(\alpha, \rvec{D}_N), \ T'(\alpha, \rvec{D}_N) \bigr) \,.
    \end{align*}
    Then it holds for $\epsilon > 0$:
    \begin{align}\label{Eq_CorPAC}
        \begin{split}
            \mathbb{P}_{\rvec{D}_N}\Bigl\{ \forall \lambda \in \Lambda, \forall \mathbb{Q} \ll \mathbb{P}_\rvec{H} \ : \
            \mathbb{E}_\mathbb{Q}[\mathcal{R}] &\le \mathbb{E}_\mathbb{Q}[\hat{\mathcal{R}}] \\
            &+ \frac{1}{\eta_1(\lambda)} 
            \bigl(
            \divergence{KL}{\mathbb{Q}}{\mathbb{P}_\rvec{H}} + \log\Bigl(\frac{\vert \Lambda \vert}{\epsilon}\Bigr)  
            - \mathbb{E}_\mathbb{Q}\bigl[\sprod{\eta'(\lambda)}{T'}\bigr] \bigr) \Bigr\} \ge 1-\epsilon \,.
        \end{split}
    \end{align}
\end{Kor}

In Section~\ref{Sec_PropertiesOfAlgorithms} we provide sufficient conditions, such that $\mathbb{E}_{\rvec{D}_N}[c(\lambda)] \le 1$ holds for all $\lambda > 0$.

\subsection{Minimization of the PAC-Bound}\label{SecMinimizationPACBound}
In this whole subsection we use $\eta$ and $T$ from Corollary~\ref{PAPER_KorGeneralizationBoundExpFamily}. We seek for $\lambda \in \Lambda$ and $\mathbb{Q} \ll \mathbb{P}_{\rvec{H}}$ that minimizes the right-hand side of the generalization bound in (\ref{Eq_CorPAC}). By factoring out $-\frac{1}{\eta_1(\lambda)}$ again, this is actually the same as:
\begin{align*}
    \inf_{\lambda \in \Lambda} \ -\frac{1}{\eta_1(\lambda)} \Bigl(\sup_{\mathbb{Q} \ll \mathbb{P}_{\rvec{H}}} 
    \mathbb{E}_\mathbb{Q}[\sprod{\eta(\lambda)}{\Tilde{T}}] 
    - \divergence{KL}{\mathbb{Q}}{\mathbb{P}_\rvec{H}} 
    - \log \Bigl( \frac{\vert \Lambda \vert}{\epsilon} \Bigr)  \Bigl),
\end{align*}
where $\Tilde{T}(\alpha, \rvec{D}_N) := \bigl(-\hat{\mathcal{R}}(\alpha, \rvec{D}_N), \ T'(\alpha, \rvec{D}_N) \bigr)$. Since $\log( \vert \Lambda \vert/\epsilon)$ is a constant, Lemma \ref{Lem_DonskerVaradhanForExponentialFamilies} shows that the term inside the brackets is actually given by $\Tilde{\kappa}(\lambda) - \log( \vert \Lambda \vert / \epsilon)$, where $\Tilde{\kappa}$ corresponds to the exponential family $\mathbb{Q}_\lambda$ built upon $\Tilde{T}, \eta$ and $h \equiv 1$. Furthermore, it shows that the optimal posterior distribution is given by the corresponding member of the exponential family (usually called the Gibbs posterior \citep{A2021}):
$$
    \frac{d\mathbb{Q}_\lambda}{d\mathbb{P}_\rvec{H}}(\alpha) = \frac{\exp(\sprod{\eta(\lambda)}{\Tilde{T}(\alpha)})}{\mathbb{E}_{\rvec{H}}[\exp(\sprod{\eta(\lambda)}{\Tilde{T}})]}\,.
$$
By denoting $F(\lambda) := -\frac{1}{\eta_1(\lambda)}( \Tilde{\kappa}(\lambda) - \log(\vert \Lambda \vert / \epsilon))$, one is left with solving the following minimization problem:
\begin{equation}\label{Eq_MinF}
    \min_{\lambda \in \Lambda} \ F(\lambda),
\end{equation}
which for $\Lambda \subset \mathbb{R}$ is one-dimensional. Under mild assumptions, one can show that $\argmin_{\lambda > 0} F(\lambda)$ lies in a bounded interval. Thus, one can control the accuracy of the solution of the minimization problem (\ref{Eq_MinF}) by the choice of $\Lambda$. The computational cost for evaluating this one-dimensional function several times is low compared to solving several minimization problems during training. 

\section{Learning Optimization Algorithms with Theoretical Convergence Guarantees} \label{Sec_PropertiesOfAlgorithms}
In this section, we consider properties of optimization algorithms, that assert the necessary condition $\mathbb{E}_{\rvec{D}_N}[c(\lambda)] \le 1$ for all $\lambda \in \Lambda$ to employ the PAC-Bayes bound from Section \ref{Sec_PAC}. Typically, this yields the functions $\eta'$ and $T'$.

\subsection{Guaranteed Convergence}\label{Sec_EnsuredConvergence}
The following convergence property is sufficient to ensure the assumptions of Theorem \ref{PAPER_PACBayesExponentialFamiliy}. Essentially, it requires the loss of the algorithm's output to be bounded. Nevertheless, it is shown in Section~\ref{SubSec_CondOnConv} that it is too restrictive to learn hyperparameters that allow for a significant acceleration compared to the standard choices from a worst-case analysis.
\begin{AssumptionGeneral}\label{AssumptionConvergence}
    There is a constant $C \ge 0$ and a measurable function $\rho: \mathcal{H} \longrightarrow \mathbb{R}_{\ge 0}$, such that it holds:
    $$
        \ell(\mathcal{A}(\alpha, \rvec{S}), \rvec{S}) \le 
        C \rho(\alpha) \ell(x^{(0)}, \rvec{S}) \qquad \forall \alpha \in \mathcal{H} \,.
    $$
\end{AssumptionGeneral}

\begin{Anm}
    The basic motivation for Assumption~\ref{AssumptionConvergence} is to take the (possibly known) convergence behaviour of an optimization algorithm into account. 
\end{Anm}

\begin{Satz}\label{Lem_AlgoWithConvergenceProperty}
    Let $N \in \mathbb{N}$ and $\rvec{D}_N$ be an i.i.d. dataset. Assume $\mathcal{A}$ satisfies Assumption \ref{AssumptionConvergence}. Further, assume that $\mathbb{E}_{\rvec{S}}\bigl[ \ell(x^{(0)}, \rvec{S})^2 \bigr] < \infty$. Define $\eta: \mathbb{R}_{>0} \longrightarrow \mathbb{R}^2$ and $T: \mathcal{H} \times \mathcal{D}_N \longrightarrow \mathbb{R}^2$ through:
    \begin{align*}
        \eta(\lambda) := \Bigl( \lambda, \ -\frac{\lambda^2}{2} \frac{C^2}{N}\mathbb{E}_{\rvec{S}}\bigl[ \ell(x^{(0)}, \rvec{S})^2 \bigr] \Bigr)\,, \ \ \  
        T (\alpha, \rvec{D}_N) := (\mathcal{R}(\alpha) - \hat{\mathcal{R}}(\alpha, \rvec{D}_N), \ \rho^2(\alpha) )\,.
    \end{align*}
    Then, for all $\lambda > 0$, it holds that: 
    $$
    \mathbb{E}_{\rvec{D}_N}[c(\lambda)] \le 1 \,.
    $$
\end{Satz}
The proof of Theorem \ref{Lem_AlgoWithConvergenceProperty} is given in the supplementary material \ref{Proof_Lem_AlgoWithConvergenceProperty}.

\subsection{Conditioning on Convergence}\label{SubSec_CondOnConv}
Most of the time, the previous approach is only able to learn hyperparameters that ensure convergence. When the considered class of functions $(\ell(\cdot, \theta))_{\theta \in \Theta}$ is that of general quadratic functions, the convergence behaviour is accurately represented by analytic quantities from a worst-case analysis. Thereby, Assumption~\ref{AssumptionConvergence} prevents "aggressive" step-size parameters that lie outside the worst-case convergence regime. This is also encoded in Assumption~\ref{AssumptionConvergence}, as C and $\rho$ are independent of $\rvec{S}$. Moreover, it can be difficult to compute them. Hence, in this section, a different approach is taken: We allow for divergence, if it only occurs in rare cases with a controllable quantity. Essentially, one only considers the loss for all the hyperparameters, where convergence occurs, as well as the probability for that. In Section~\ref{SubSec_GuaranteeConvProb}, we develop a technique that allows the user to control this probability. Clearly, a higher convergence guarantee trades for convergence speed. To the best of our knowledge, the following way of dealing with the rare, unwanted case is completely new. 

\begin{Def}
    The \emph{convergence set} is defined as the set-valued mapping $C: \mathcal{H} \rightrightarrows \Theta$ with
    $$
        C(\alpha) := \{ \theta \in \Theta \ \vert \ \ell(\mathcal{A}(\alpha, \theta), \theta) \le \ell(x^{(0)}, \theta)\} \,.
    $$
\end{Def}

\begin{Anm}
    Other definitions of the convergence set are possible and the concrete choice will influence the resulting PAC-bound. For proving the result, the essential property is that the loss after application of $\mathcal{A}$ can be bounded, such that the (conditional) expectation is finite.
\end{Anm}

For every $\alpha \in \mathcal{H}$, the set $C(\alpha)$ is measurable, as the map $\theta \mapsto \ell(\mathcal{A}(\alpha, \theta), \theta) - \ell(x^{(0)}, \theta)$ is measurable. Nevertheless, we have to make the following assumption:
\begin{AssumptionGeneral}\label{Ass_p_Meas}
    The map 
    $
    p: \mathcal{H} \longrightarrow [0,1], \ \ \alpha \mapsto p(\alpha) := \mathbb{P}_\rvec{S}[C(\alpha)]
    $
    is measurable.
\end{AssumptionGeneral}

\begin{Anm}
    Although at first Assumption \ref{Ass_p_Meas} seems very restrictive, it is actually very mild in the sense that in our use case one can always find a measurable version of $p$. For more details see Lemma \ref{Proof_Rem_Meas_p}, whose assumptions are always met in our setting.
\end{Anm}

\noindent
Now we define the convergence risk as the expect loss conditioned on the convergence of the algorithm:
\begin{Def}
    The \emph{convergence risk} is defined as the conditional expectation of the loss given $C(\alpha)$:
    \begin{align*}
        \mathcal{R}_c(\alpha) 
        := \mathbb{E}_\rvec{S}[\ell(\mathcal{A}(\alpha, \rvec{S}), \rvec{S}) \ \vert \ C(\alpha)] 
        = \begin{cases}
            \frac{1}{p(\alpha)} \mathbb{E}_\rvec{S}[\mathds{1}_{C(\alpha)}(\rvec{S}) \ell(\alpha, \rvec{S})], &\text{if $p(\alpha) > 0$} \,; \\
            0, &\text{else} \,.
        \end{cases}
    \end{align*}
    Given a dataset $\rvec{D}_N = (\rvec{S}_1,...,\rvec{S}_N)$, the \emph{empirical convergence risk} is defined as:
    $$
        \hat{\mathcal{R}}_c(\alpha, \rvec{D}_N) := \frac{1}{p(\alpha)} \frac{1}{N} \sum_{i=1}^N \mathds{1}_{C(\alpha)}(\rvec{S}_i) \ell(\alpha, \rvec{S}_i) \,.
    $$
\end{Def}
The following theorem is a generalization of Theorem \ref{Lem_AlgoWithConvergenceProperty}. The proof is given in the supplementary material \ref{Proof_PacBayesConditioned}.

\begin{Satz} \label{Lem_PacBayesConditioned}
    Assume that $\mathbb{P}_\rvec{H}\{ p > 0 \} = 1$ and $\mathbb{E}_\rvec{S}[\ell(x^{(0)}, \rvec{S})^2] < \infty$. Define $\eta: \mathbb{R}_{> 0} \longrightarrow \mathbb{R}^2$ and $T: \mathcal{H} \times D_N \longrightarrow \mathbb{R}^2$ through
    \begin{align*}
        \eta(\lambda) := \Bigl( \lambda, \ -\frac{\lambda^2}{2} \frac{1}{N}\mathbb{E}_\rvec{S}\Bigl[ \ell(x^{(0)}, \rvec{S})^2\Bigr]  \Bigr)\,, \ \ \
        T(\alpha, \rvec{D}_N) := \Bigl( \mathcal{R}_c(\alpha) - \hat{\mathcal{R}}_c(\alpha, \rvec{D}_N), \ \frac{1}{p(\alpha)^2} \Bigr) \,.
    \end{align*}
    Then, for all $\lambda > 0$, it holds that: 
    $$
    \mathbb{E}_{\rvec{D}_N}[c(\lambda)] \le 1 \,.
    $$
\end{Satz}

\begin{Anm}
    $\mathbb{P}_\rvec{H}\{ p > 0 \} = 1$ says that, under the prior, the algorithm should not diverge exclusively.
\end{Anm}

\subsection{Guarantee of Convergence with High Probability}\label{SubSec_GuaranteeConvProb}
In the previous approach, care has to be taken in the choice of the prior $\mathbb{P}_\rvec{H}$: Constructing the prior in a way that minimizes the upper bound as much as possible can lead to the case where a high convergence probability is neglected, i.e., the algorithm converges only on a small subset of the parameters and for them especially fast, because the term $\frac{1}{p(\alpha)}$ might not compensate for the smaller convergence risk. Thus, if a certain convergence probability $\epsilon_{conv}$ has to be satisfied, one has to ensure this in another way. We propose to use a direct consequence of absolute continuity:
\begin{Lemma}
    Let $\epsilon_{conv} \in [0,1]$ and $\mathbb{P}_\rvec{H}$ be such that $\mathbb{P}_\rvec{H} \{ p < \epsilon_{conv} \} = 0$. Then it holds for every $\mathbb{Q} \ll \mathbb{P}_{\rvec{H}}$:
    $$
    \mathbb{Q} \{ p < \epsilon_{conv} \} = 0 \,.
    $$
\end{Lemma}
Though the proof is trivial, this lemma has a very important consequence, which we want to emphasize here: If one can guarantee that a required property is satisfied for the prior, it will be preserved during the PAC-Bayes learning process, i.e., if the prior only puts mass on hyperparameters that ensure a certain convergence probability, also the posterior will allow only hyperparameters that ensure the same convergence probability. Thus, ensuring a convergence probability will be part of the construction of the prior. 

\section{Experiments}\label{Sec_Experiments}
In all experiments, we use $n = 50$ and a quadratic loss function for which we can choose the smallest and largest eigenvalue, i.e., a loss of the form $\frac{1}{2}\Vert Ax - b \Vert^2$. As optimization algorithms we unroll either the Heavy-ball (\ref{Eq_UpdateHB}) or Gradient Descent update step for a fixed number of iterations. In the case of Gradient Descent we learn the (constant) step-size, and in the case of heavy-ball we learn the step-size and the extrapolation parameter (both constant). Note that all results are created with a single sample from the posterior and do not show the expected value under the posterior. The experiments are a proof-of-concept for our theory in an easily controllable setting. Actually, our theory does not require a quadratic, in fact not even convex, loss function. More details about the learning procedure are given in Appendix~\ref{Sec_DescriptionLearningProcedure}.

\begin{figure}[t!]
\centering
\tikz{
    \node (A) {\includegraphics{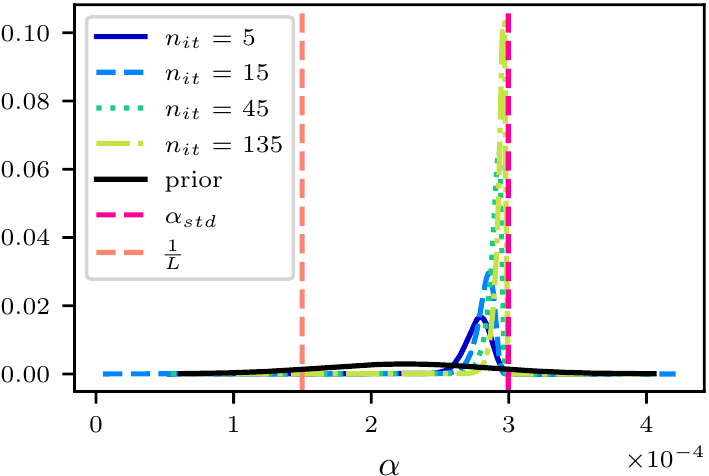}};
    \node[below right=-5.08cm and 5mm of A] (B) {\includegraphics{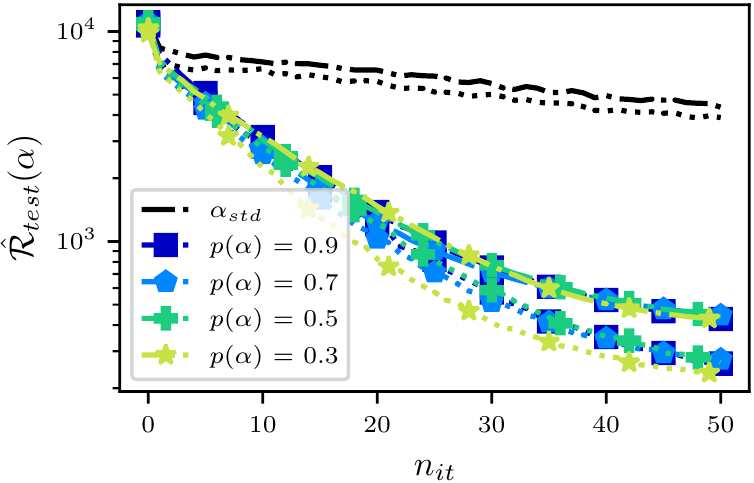}};
    \draw (A.south) node[below=3mm, text width=0.45\linewidth, align=justify] {\caption{\label{fig:PosteriorConverges}Posterior for an increasing number of iterations: The initial prior is chosen as a Gaussian centered at $\frac{1}{2}(\frac{1}{L} + \frac{2}{L})$. The posterior distributions for $N_{it} \in \{5, 15, 45, 135\}$ are shown. For an increasing number of iterations the posterior puts increasingly more mass close to $\alpha_{std} = \frac{2}{L+\mu}$.}};
    \draw (B.south) node[below=2.1mm, text width=0.45\linewidth, align=justify] {\caption{\label{fig:CondOnConv}Test loss over the iterations: The black lines are for the standard choices of the hyperparameters. The empirical mean is given by the dashed line, and the median by the dotted one. The other lines show the test loss for $p(\alpha) \in \{0.9, 0.7, 0.5, 0.3\}$. By excluding the worst-case, one can accelerate the optimization procedure significantly.}};
}
\end{figure}

\subsection{Convergence of the Posterior}
The first experiment considers the posterior distribution over the step-size parameter of Gradient Descent. The parameter $\rvec{S}$ is given by the right-hand side $b$ of the quadratic problem, i.e., all problems have the same strong convexity parameter $\mu$ and the same smoothness parameter $L$ (smallest and largest eigenvalue of $A^T A$). We use $N_{train} = 200$ and build the exponential family with $\eta$ and $T$ from Section~\ref{Sec_EnsuredConvergence}, i.e., convergence is guaranteed during learning. We are interested in how the posterior distribution evolves for an increasing number of iterations of the algorithm. Since it is known that $\alpha_{std} = \frac{2}{L+\mu}$ yields the optimal rate in the worst-case \citep{N2018}, one would expect that the posterior puts increasingly more mass onto step-sizes close to $\alpha_{std}$. Figure~\ref{fig:PosteriorConverges} confirms this intuition. Also, it shows that Assumption~\ref{AssumptionConvergence} prohibits step-sizes larger than $\frac{2}{L}$, which could lead to divergence easily. 

\subsection{Conditioning on Convergence}\label{ExpCondOnConv}
Here, the parameters $\rvec{S}$ of $\ell$ are given by the quadratic matrix and the right-hand side, i.e., the problems have a differing strong convexity parameter $\mu$ and smoothness parameter $L$. We sample these from a uniform distribution over $[\mu_-, \mu_+]$ and $[L_-, L_+]$. This simulates a situation where these parameters can only be estimated roughly. We use the Heavy-ball method for 50 iterations. The standard choice for the hyperparameters are given by $\tau_{std} = \Bigl(\frac{2}{\sqrt{L_+} + \sqrt{\mu_-}}\Big)^2$ and $\beta_{std} = \Bigl( \frac{\sqrt{L_+} - \sqrt{\mu_-}}{\sqrt{L_+} + \sqrt{\mu_-}}\Bigr)^2
$ \citep{N2018}. We use $N_{prior} = 100$, $N_{train} = 100$ and $N_{test} = 200$. Figure~\ref{fig:CondOnConv} shows the convergence behaviour for different convergence guarantees. As one can see, excluding the worst-case ($\epsilon_{conv} \ge 0.9$) leads to a significantly better convergence result. However, a further decrease of the convergence guarantee does not lead to a further acceleration. This does not match the expected behaviour, yet is explained by the next experiment. 

\begin{figure}[t!]
\centering
\tikz{
    \node (A) {\includegraphics{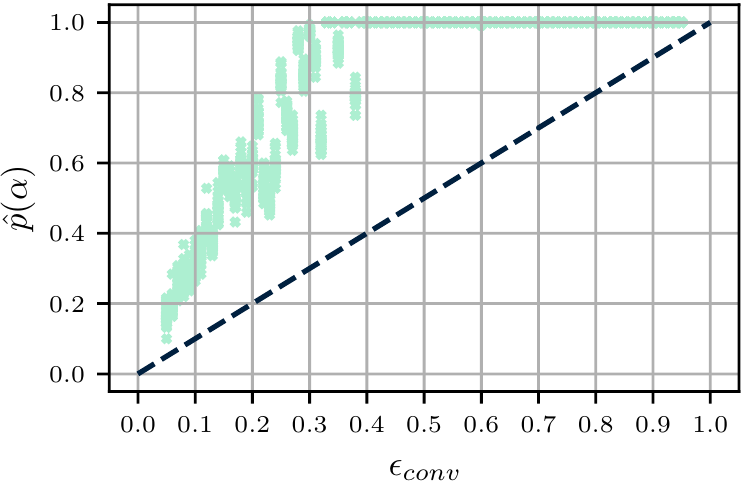}};
    \node[below right=-5.15cm and 3mm of A] (B) {\includegraphics{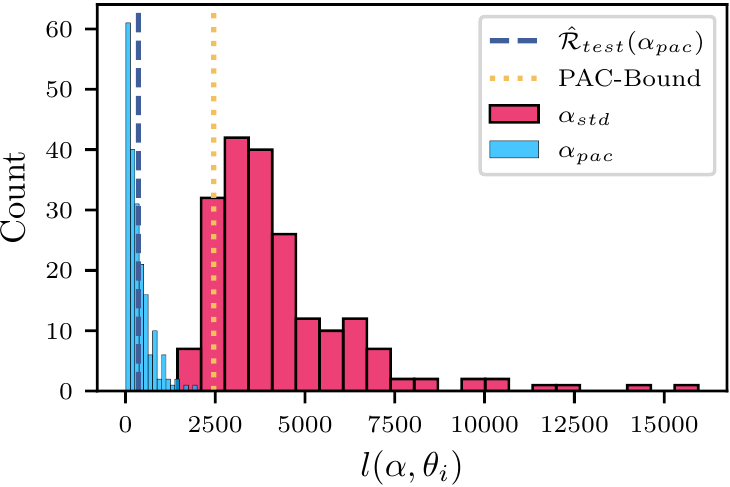}};
    \draw (A.south) node[below=3mm, text width=0.45\linewidth, align=justify] {\caption{\label{fig:EmpConvProb}Empirical convergence probability: The dashed diagonal indicates the user-specified convergence probability. Each cross represents the empirical convergence probability on a separate test set. All empirical convergence probabilities lie well above the diagonal, i.e., the algorithm indeed ensures the user-specified convergence probability.}};
    \draw (B.south) node[below=2.5mm, text width=0.45\linewidth, align=justify] {\caption{\label{fig:PAC_bound}Test loss as histogram: The blue thin bars represent the learned hyperparameters and the red thick bars the ones from a worst-case analysis.  The vertical lines represent the empirical mean for the learned hyperparameters (blue dashed) and the corresponding PAC-bound (orange dotted). The learned hyperparameters clearly outperform the standard ones, yet the PAC-bound is not perfectly tight.}};
}
\end{figure}

\subsection{Ensuring a Certain Convergence Probability}
We use the same setup as in Section~\ref{ExpCondOnConv} and investigate the empirical convergence probability on several test sets. We use $N_{prior} = 100$, $N_{train} = 100$ and $25$ test sets of size $N_{test} = 250$ per user-specified convergence probability to estimate the true convergence probability of the algorithm. Note that we use the same datasets for all different convergence probabilities, i.e., we create them beforehand. We use the standard estimator for binomial distributions as empirical estimate for the convergence probability, i.e., $\hat{p}(\alpha) = \frac{N_{conv}}{N_{test}}$. Figure~\ref{fig:EmpConvProb} shows the result of this experiment: All empirical convergence probabilities lie well above the diagonal, i.e., the algorithm indeed ensures the user-specified convergence probability. However, one can also see that it clearly favors a higher convergence probability than necessary, which can hinder the performance and explains the somewhat unexpected behaviour in the previous experiment. As indicated by the theory, this behaviour is probably due to our construction of the prior distribution.

\subsection{Evaluation of the PAC-Bound}
This experiment looks at the tightness of the PAC-bound. We adopt the setting from Section~\ref{ExpCondOnConv}. Based on the previous experiment, we choose $\epsilon_{conv} = 0.9$ as convergence guarantee. Further, we use $N_{prior} = 200$, $N_{train} = 1000$ and $N_{test} = 200$. The training dataset is chosen larger than before, since the PAC-bound is not yet very tight for small datasets. Figure~\ref{fig:PAC_bound} shows the resulting losses on the test set as histogram plot, as well as the empirical mean and the PAC-Bayes bound. One can clearly see the improved performance of Heavy-ball with the learned hyperparameters. Further, one can see that the PAC-bound is not perfectly tight, however provides a good estimate of the true mean. 

\section{Conclusion}
We presented a general PAC-Bayes theorem based on exponential families, which allows for a unified implementation of the learning framework. We applied this framework to the setting of learning-to-optimize and derived properties, under which the theorem is applicable to a given algorithm. Further, we provided a principled way to exclude unwanted cases by using conditional expectations and showed how to preserve necessary properties during learning. We believe that both approaches can be of great interest even aside the setting of learning-to-optimize. Finally, we provided a proof-of-concept of our theory on several experiments. 
\paragraph{Limitations.} 
We mainly see four limitations of our work: First, a theoretical guarantee to find the global minimum in (\ref{Eq_MinF}) is still missing. Second, the construction of the prior is difficult and time-consuming. Third, we expect similar scaling problems for high-dimensional hyperparameters as with other probabilistic methods. And fourth, the trade-off between convergence speed and convergence probability is partly rather conservative. These problems are very related and will be addressed in future work. 

\subsubsection*{Acknowledgements}
We acknowledge funding by the German Research Foundation under Germany’s Excellence Strategy – EXC number 2064/1 – 390727645 and the project DFG OC 150/5-1. Furthermore, we thank J. Fadili and one of the anonymous reviewers for an important hint for the extension of the union bound argument to compact
sets.

\newpage
\appendix

\section{Missing Proofs}\label{Sec_MissingProofs}

\subsection{Proof of Lemma \ref{Lem_DonskerVaradhanForExponentialFamilies}} \label{Proof_Lem_DonskerVaradhanForExponentialFamilies}
\begin{proof}
Recall that $\kappa(\lambda) = \log \Bigl( \int  h \exp\Bigl(\sprod{\eta(\lambda)}{T}\Bigr) d \mathbb{P}_\rvec{H} \Bigr)$ and $A(\lambda) = c(\lambda)^{-1}$. We have to show: 
\begin{itemize}
    \item[1)] $\kappa(\lambda) = \sup_{\mathbb{Q} \ll \mathbb{P}_\rvec{H}} \mathbb{E}_\mathbb{Q}\bigl[ \sprod{\eta(\lambda)}{T} + \log(h) \bigr] - \divergence{KL}{\mathbb{Q}}{\mathbb{P}_\rvec{H}}$.
    \item[2)] The supremum is attained at $\mathbb{Q}_\lambda$.
\end{itemize} 
For this, we first show $\kappa(\lambda) \ge \mathbb{E}_\mathbb{Q}\bigl[ \sprod{\eta(\lambda)}{T} + \log(h) \bigr] - \divergence{KL}{\mathbb{Q}}{\mathbb{P}_\rvec{H}}$ for an arbitrary $\mathbb{Q} \ll \mathbb{P}_\rvec{H}$ and then equality for $\mathbb{Q}_\lambda$.
Thus, let $\mathbb{Q} \ll \mathbb{P}_\rvec{H}$ and denote by $\frac{d \mathbb{Q}}{d \mathbb{P}_\rvec{H}}$ its Radon-Nikodym derivative w.r.t. $\mathbb{P}_\rvec{H}$. Then it holds:
\begin{align*}
    \mathbb{E}_\mathbb{Q}\Bigl[ \sprod{\eta(\lambda)}{T} + \log(h) \Bigr] - \divergence{KL}{\mathbb{Q}}{\mathbb{P}_\rvec{H}}
    &= \int \sprod{\eta(\lambda)}{T} + \log(h) - \log\Bigl(\frac{d \mathbb{Q}}{d \mathbb{P}_\rvec{H}}\Bigr) \ d\mathbb{Q} \\
    &= \int \log \Bigl( \frac{h}{\frac{d \mathbb{Q}}{d \mathbb{P}_\rvec{H}}} \exp\Bigl(\sprod{\eta(\lambda)}{T}\Bigr) \Bigr) \ d\mathbb{Q} \,.
    \intertext{Since the logarithm is concave, by Jensen's inequality this can be bounded by:}
    &\le \log \Bigl( 
    \int  \frac{h}{\frac{d \mathbb{Q}}{d \mathbb{P}_\rvec{H}}} \exp\Bigl(\sprod{\eta(\lambda)}{T}\Bigr) \ d\mathbb{Q} 
    \Bigr) \\
    &= \log \Bigl( 
    \int  h \exp\Bigl(\sprod{\eta(\lambda)}{T}\Bigr) \ d\mathbb{P}_\rvec{H} 
    \Bigr) \\
    &= \kappa(\lambda) \,.
\end{align*}
It remains to show the equality for $\mathbb{Q}_\lambda$:
\begin{align*}
    \divergence{KL}{\mathbb{Q}_\lambda}{\mathbb{P}_\rvec{H}} 
    &= \int \log \Bigl( \frac{d\mathbb{Q}_\lambda}{d \mathbb{P}_\rvec{H}} \Bigr) \ d \mathbb{Q}_\lambda \\
    &= \int \log \Bigl( h A(\lambda) \exp(\sprod{\eta(\lambda)}{T}) \Bigr) \ d \mathbb{Q}_\lambda \\
    &= \int \log (h) + \sprod{\eta(\lambda)}{T} \ d \mathbb{Q}_\lambda + \log(A(\lambda))\\
    &= \mathbb{E}_{\mathbb{Q}_\lambda}\Bigl[\log (h) + \sprod{\eta(\lambda)}{T}\Bigr] - \log(c(\lambda)) \\
    &= \mathbb{E}_{\mathbb{Q}_\lambda}\Bigl[\log (h) + \sprod{\eta(\lambda)}{T}\Bigr] - \kappa(\lambda),
\end{align*}
which yields:
$$
\kappa(\lambda) = \mathbb{E}_{\mathbb{Q}_\lambda}\Bigl[\log (h) + \sprod{\eta(\lambda)}{T}\Bigr] - \divergence{KL}{\mathbb{Q}_\lambda}{\mathbb{P}_\rvec{H}} \,.
$$
Thus, the supremum is attained at $\mathbb{Q}_\lambda$. This concludes the proof.
\end{proof}

\subsection{Proof of Theorem \ref{PAPER_PACBayesExponentialFamiliy}} \label{Proof_PAPER_PACBayesExponentialFamiliy}
\begin{proof}
We will use $c(\lambda)$ and $\kappa(\lambda)$ as a short-hand for $c(\lambda, \rvec{D}_N)$ and $\kappa(\lambda, \rvec{D}_N)$ respectively. $c(\lambda)$ is a non-negative random variable and $\exp$ is a monotonically increasing function. Thus, since $\mathbb{E}_{\rvec{D}_N}[c(\lambda)] \le 1$ for all $\lambda \in \Lambda$, one gets for every $\lambda \in \Lambda$ from Markov's inequality for every $s \in \mathbb{R}$:
$$
    \mathbb{P}_{\rvec{D}_N} \Bigl\{ c(\lambda) > \exp(s) \Bigr\} \le \frac{\mathbb{E}_{\rvec{D}_N}[c(\lambda)]}{\exp(s)} \le \exp(-s) \,.
$$
Since $c(\lambda) > \exp(s) \ \Leftrightarrow \ \kappa(\lambda) = \log(c(\lambda)) > s$, this is the same as:
$$
    \mathbb{P}_{\rvec{D}_N} \Bigl\{\kappa(\lambda) > s \Bigr\} \le \exp(-s) \,.
$$
This implies by the union-bound argument, that:
$$
    \mathbb{P}_{\rvec{D}_N} \Bigl\{\sup_{\lambda \in \Lambda} \ \kappa(\lambda) > s \Bigr\} 
    = 
    \mathbb{P}_{\rvec{D}_N} \Bigl\{ \bigcup_{\lambda \in \Lambda} \{ \kappa(\lambda) > s \} \Bigr \} 
    \le 
    \sum_{\lambda \in \Lambda} \mathbb{P}_{\rvec{D}_N} \Bigl\{ \kappa(\lambda) > s  \Bigr \} 
    \le \vert \Lambda \vert \exp(-s) \,.
$$
Inserting $s = \log\Bigl( \frac{\vert \Lambda \vert}{\epsilon} \Bigr)$ gives:
$$
    \mathbb{P}_{\rvec{D}_N} \Bigl\{\sup_{\lambda \in \Lambda} \ \kappa(\lambda) > \log\Bigl( \frac{\vert \Lambda \vert}{\epsilon} \Bigr) \Bigr\} 
    \le 
    \epsilon \,.
$$
Hence, the complementary event yields:
$$
    \mathbb{P}_{\rvec{D}_N} \Bigl\{ \sup_{\lambda \in \Lambda} \kappa(\lambda) \le \log\Bigl( \frac{\vert \Lambda \vert }{\epsilon} \Bigr) \Bigr\} \ge 1-\epsilon \,.
$$
Using $\kappa(\lambda) = \sup_{\mathbb{Q} \ll \mathbb{P}_\rvec{H}} \mathbb{E}_\mathbb{Q}\bigl[ \sprod{\eta(\lambda)}{T} + \log(h) \bigr] - \divergence{KL}{\mathbb{Q}}{\mathbb{P}_\rvec{H}}$ then gives:
$$
    \mathbb{P}_{\rvec{D}_N} \Bigl\{ \sup_{\lambda \in \Lambda} \  \sup_{\mathbb{Q} \ll \mathbb{P}_\rvec{H}} \mathbb{E}_\mathbb{Q}[\log(h) + \sprod{\eta(\lambda)}{T}] - \divergence{KL}{\mathbb{Q}}{\mathbb{P}_\rvec{H}} \le \log\Bigl( \frac{\vert \Lambda \vert}{\epsilon} \Bigr) \Bigr\} \ge 1-\epsilon \,.
$$
Rearranging and reformulating then yields the result:
$$
    \mathbb{P}_{\rvec{D}_N} \Bigl\{ \forall \lambda \in \Lambda, \ \forall \mathbb{Q} \ll \mathbb{P}_\rvec{H} \ : \ \mathbb{E}_\mathbb{Q}[\log(h) + \sprod{\eta(\lambda)}{T}] \le  \divergence{KL}{\mathbb{Q}}{\mathbb{P}_\rvec{H}} + \log \Bigl( \frac{\vert \Lambda \vert}{\epsilon} \Bigr) \Bigr\} \ge 1-\epsilon \,.
$$
\end{proof}

\subsection{Proof of Theorem \ref{Lem_AlgoWithConvergenceProperty}}\label{Proof_Lem_AlgoWithConvergenceProperty}
\begin{proof}
We use the following short-hand notation:
$$
    \rvec{L}(\alpha) := \ell(\mathcal{A}(\alpha, \rvec{S}), \rvec{S}), \ \ 
    \rvec{L}_i(\alpha) := \ell(\mathcal{A}(\alpha, \rvec{S}_i), \rvec{S}_i),
    \ \ 
    \rvec{L}_0 := \ell(x^{(0)}, \rvec{S}) \,.
$$
By the i.i.d. assumption, one can write for every fixed $\alpha \in \mathcal{H}$:
\begin{align*}
    &\mathbb{E}_{\rvec{D}_N} \Bigl[ \exp\Bigl( \lambda (\mathcal{R}(\alpha) - \hat{\mathcal{R}}(\alpha, \rvec{D}_N)) \Bigr)\Bigr] 
    = \mathbb{E}_{\rvec{D}_N} \Bigl[ \exp\Bigl( 
     -\frac{\lambda}{N} \sum_{i=1}^N \Bigl(\rvec{L}_i - \mathbb{E}_{\rvec{S}}[\rvec{L} ]  \Bigr) \Bigr)\Bigr] \\
    &= \mathbb{E}_{\rvec{D}_N}\Bigl[ \prod_{i=1}^N \exp\Bigl( -\frac{\lambda}{N} (\rvec{L}_i - \mathbb{E}_\rvec{S}[\rvec{L}]) \Bigr)\Bigr]
     \overset{iid}{=}\prod_{i=1}^N \mathbb{E}_{\rvec{S}} 
     \Bigl[ \exp\Bigl( 
     -\frac{\lambda}{N} \Bigl(\rvec{L} - \mathbb{E}_{\rvec{S}}[\rvec{L}] \Bigr)\Bigr)\Bigr] \,.
     \intertext{Since the loss-function is non-negative and $\mathcal{A}$ satisfies the convergence property, one gets that $\rvec{L}$ is a non-negative random variable with finite second-moment:
     \begin{align*}
         \mathbb{E}_{\rvec{S}}[\rvec{L}^2] &= 
        \mathbb{E}_{\rvec{S}}\Bigl[ \ell\bigl(\mathcal{A}(\alpha, \rvec{S}), \rvec{S} \bigr)^2\Bigr] \\
        &\le C^2 \rho(\alpha)^2 \mathbb{E}_{\rvec{S}}\Bigl[\ell\bigl(x^{(0)}, \rvec{S} \bigr)^2\Bigr]
        = C^2 \rho(\alpha)^2 \mathbb{E}_{\rvec{S}}[\rvec{L}_0^2] \,.
     \end{align*}
     Thus, by lemma \ref{LemmaSubGaussianLowerTail}, one gets the following bound:}
     &\le \prod_{i=1}^N  
     \exp\Bigl( 
     \frac{\lambda^2}{2N^2} \mathbb{E}_{\rvec{S}} [ \rvec{L}^2 ] \Bigr) 
     = \exp\Bigl( 
     \frac{\lambda^2}{2N} \mathbb{E}_{\rvec{S}} [\rvec{L}^2] \Bigr) \,.
     \intertext{Since the exponential function is monotonically increasing, by the convergence property this can again be bounded by:}
     &\le \exp\Bigl( 
     \frac{\lambda^2}{2N} C^2 \rho(\alpha)^2 \mathbb{E}_{\rvec{S}} [\rvec{L}_0^2] \Bigr) \,.
\end{align*}
Thus, for any $\alpha \in \mathcal{H}$ one arrives at the following inequality:
$$
\mathbb{E}_{\rvec{D}_N} \Bigl[ \exp\Bigl( \lambda (\mathcal{R}(\alpha) - \hat{\mathcal{R}}(\alpha, \rvec{D}_N)) \Bigr)\Bigr] \le 
\exp\Bigl( \frac{\lambda^2}{2N} C^2 \rho(\alpha)^2 \mathbb{E}_{\rvec{S}} [\rvec{L}_0^2] \Bigr) \,.
$$
Since the right-hand side is a constant w.r.t. $\mathbb{P}_{\rvec{D}_N}$, rearranging terms gives:
$$
\mathbb{E}_{\rvec{D}_N} \Bigl[ \exp\Bigl( \lambda (\mathcal{R}(\alpha) - \hat{\mathcal{R}}(\alpha, \rvec{D}_N)) - \frac{\lambda^2}{2} \frac{C^2}{N} \rho(\alpha)^2 \mathbb{E}_{\rvec{S}} [\rvec{L}_0^2] \Bigr) \Bigr] 
\le 1 \,.
$$
By integrating both sides with respect to $\mathbb{P}_\rvec{H}$ and using Fubini's theorem (note that this is possible, since $\mathbb{P}_\rvec{H}$ is independent of $\rvec{D}_N$), one gets:
$$
\mathbb{E}_{\rvec{D}_N} \Bigl[ 
\mathbb{E}_\rvec{H} \Bigl[ \exp\Bigl( \lambda (\mathcal{R}(\rvec{H}) - \hat{\mathcal{R}}(\rvec{H}, \rvec{D}_N)) \Bigr) - \frac{\lambda^2}{2} \frac{C^2}{N} \rho(\rvec{H})^2 \mathbb{E}_{\rvec{S}} [\rvec{L}_0^2] \Bigr] \Bigr]
\le 1 \,.
$$
Inserting the definition of $\eta$ and $T$ now gives:
$$
\mathbb{E}_{\rvec{D}_N} \Bigl[ 
\mathbb{E}_\rvec{H} \Bigl[ \exp\Bigl( \sprod{\eta(\lambda)}{T(\rvec{H}, \rvec{D}_N)} \Bigr) \Bigr] \Bigr]
\le 1 \,.
$$
By definition of the Laplace transform, this is the same as:
$$
\mathbb{E}_{\rvec{D}_N}[c(\lambda, \rvec{D}_N)] \le 1 \,.
$$
\end{proof}

\subsection{Proof of Theorem \ref{Lem_PacBayesConditioned}}\label{Proof_PacBayesConditioned}
\begin{proof}
The proof is very similar to the proof of lemma 5.1 and basically follows the same line of argumentation. We use $\ell_c(\alpha, \theta) := \mathds{1}_{C(\alpha)}(\theta) \ell (\alpha, \theta)$ as short-hand and call this the convergence loss. First, consider $\alpha \in \mathcal{H}$ fixed with $p(\alpha) > 0$. Then it holds:
\begin{align*}
    &\mathbb{E}_{\rvec{D}_N} \Bigl[ \exp(\lambda(\mathcal{R}_c(\alpha) - \hat{\mathcal{R}}_c(\alpha, \rvec{D}_N))) \Bigr] 
    = \mathbb{E}_{\rvec{D}_N} \Bigl[ \exp\Bigl(-\frac{\lambda}{N p(\alpha)} \sum_{i=1}^N \Bigl(\ell_c(\alpha, \rvec{S}_i) - \mathbb{E}_\rvec{S}[\ell_c(\alpha, \rvec{S})] \Bigr)\Bigr) \Bigr] \\
    &= \mathbb{E}_{\rvec{D}_N} \Bigl[ \prod_{i=1}^N \exp\Bigl(-\frac{\lambda}{N p(\alpha)} \Bigl( \ell_c(\alpha, \rvec{S}_i) - \mathbb{E}_\rvec{S}[\ell_c(\alpha, \rvec{S})] \Bigr)\Bigr)\Bigr] \,.
    \intertext{Since $\rvec{D}_N$ is assumed to be i.i.d., this is the same as:}
    &= \prod_{i=1}^N \mathbb{E}_{\rvec{S}} \Bigl[  \exp\Bigl(-\frac{\lambda}{N p(\alpha)} \Bigl( \ell_c(\alpha, \rvec{S}) - \mathbb{E}_\rvec{S}[\ell_c(\alpha, \rvec{S})] \Bigr)\Bigr) \Bigr] \,.
    \intertext{Since the convergence loss is non-negative and has a finite second-moment (since $\mathbb{E}_\rvec{S}[\ell_c(\alpha, \rvec{S})^2] \le \mathbb{E}_\rvec{S}[\ell(x^{(0)}, \rvec{S})^2] < \infty$), by lemma 3.14 this can be bounded by:}
    &\le \prod_{i=1}^N \exp\Bigl(\frac{\lambda^2}{2 N^2 p(\alpha)^2} \mathbb{E}_\rvec{S}\Bigl[\ell_c(\alpha, \rvec{S})^2\Bigr]\Bigr) 
    = \exp\Bigl(\frac{\lambda^2}{2 N p(\alpha)^2} \mathbb{E}_\rvec{S}\Bigl[\ell_c(\alpha, \rvec{S})^2\Bigr]\Bigr) \,.
    \intertext{By definition of the convergence set, this can in turn be bounded by:}
    &\le \exp\Bigl(\frac{\lambda^2}{2 N p(\alpha)^2} \mathbb{E}_\rvec{S}\Bigl[\mathds{1}_{C(\rvec{S})} \ell(x^{(0)}, \rvec{S})^2\Bigr]\Bigr)  
    \le \exp\Bigl(\frac{\lambda^2}{2 N p(\alpha)^2} \mathbb{E}_\rvec{S}\Bigl[\ell(x^{(0)}, \rvec{S})^2\Bigr]\Bigr) \,.
\end{align*}
Thus, one gets $\mathbb{P}_\rvec{H}$-a.s.:
$$
\mathbb{E}_{\rvec{D}_N} \Bigl[ \exp(\lambda(\mathcal{R}_c(\alpha) - \hat{\mathcal{R}}_c(\alpha, \rvec{D}_N))) \Bigr] \le \exp\Bigl(\frac{\lambda^2}{2 N p(\alpha)^2} \mathbb{E}_\rvec{S}\Bigl[\ell(x^{(0)}, \rvec{S})^2\Bigr]\Bigr) \,.
$$
Since the right-hand side is independent of $\rvec{D}_N$, this is equivalent to:
$$
\mathbb{E}_{\rvec{D}_N} \Bigl[ \exp\Bigl(\lambda(\mathcal{R}_c(\alpha) - \hat{\mathcal{R}}_c(\alpha, \rvec{D}_N)) - \frac{\lambda^2}{2 N p(\alpha)^2} \mathbb{E}_\rvec{S}\Bigl[\ell(x^{(0)}, \rvec{S})^2\Bigr] \Bigr) \Bigr] \le 1 \,.
$$
Since $\mathbb{P}_\rvec{H} \Bigl\{ p(\rvec{H}) > 0 \Bigr \} = 1$, one can integrate both sides w.r.t. $\mathbb{P}_\rvec{H}$. Furthermore, since $\mathbb{P}_\rvec{H}$ is independent of $\rvec{D}_N$, one can use Fubini's theorem to get:
$$
    \mathbb{E}_{\rvec{D}_N} \Bigl[ \mathbb{E}_{\rvec{H}} \Bigl[ 
    \exp\Bigl(\lambda(\mathcal{R}_c(\rvec{H}) - \hat{\mathcal{R}}_c(\rvec{H}, \rvec{D}_N)) - \frac{\lambda^2}{2 N p(\alpha)^2} \mathbb{E}_\rvec{S}\Bigl[\ell(x^{(0)}, \rvec{S})^2\Bigr] \Bigr) \Bigr] \Bigr] \le 1 \,.
$$
Using the definition of $\eta$ and $T$, this is the same as:
$$
    \mathbb{E}_{\rvec{D}_N} \Bigl[ \mathbb{E}_{\rvec{H}} \Bigl[ \exp(\sprod{\eta(\lambda)}{T(\rvec{H}, \rvec{D}_N)}) \Bigr] \Bigr] \le 1,
$$
which is equivalent to:
$$
    \mathbb{E}_{\rvec{D}_N} \Bigl[ c(\lambda, \rvec{D}_N) \Bigr] \le 1 \,.
$$
\end{proof}

\section{Description of the Learning Procedure}\label{Sec_DescriptionLearningProcedure}
In this section we provide further details about the implementation.

\subsection{General Setup}
We use $n = 50$ as dimension of the optimization problem, 50 iterations of the algorithm, $x^{(0)} = 0 \in \mathbb{R}^n$ as initialization and $\epsilon = 0.01$ as threshold in the PAC-bound. For every implementation of the parametric optimization algorithm $\mathcal{A}$, we specify all (learnable) hyperparameters in a named dictionary, such that we can match all involved quantities like corresponding priors during learning by their unique names. Furthermore, since we perform first-order gradient-based optimization, we implement every algorithm in the form $\mathcal{A}(x^{(0)}, \theta, \nabla_x \ell, \alpha, n_{it})$, where $\nabla_x \ell$ denotes the gradient of $\ell$ w.r.t. $x$ as function of $\theta$. Through this, the following learning procedure can be applied to all tested algorithms in the same way.

\subsection{Creation of the Parametric Problems}
\paragraph{Fixed strong convexity and smoothness parameters.} We create the matrix $A \in \mathbb{R}^{n \times n}$ randomly, where every entry is created by sampling an integer in $\{-10, ..., 10\}$ uniformly at random and then adding noise from a standard normal distribution. This matrix is fixed across the different instances of the problem, such that all problems have the same strong convexity and smoothness parameter. For the right-hand side $b \in \mathbb{R}^n$, which in this case is the only parameter of the parametric optimization problem, we first create a mean $m$ and a covariance matrix $\Sigma$ by sampling every entry uniformly at random in $\{-5, ..., 5\}$ (and updating $\Sigma \leftarrow \Sigma^T \Sigma$ to make it positive definite), and then we sample $N_{prior} + N_{train} + N_{test}$ right-hand sides from the multivariate normal distribution $\mathcal{N}(m, \Sigma)$.

\paragraph{Varying strong convexity and smoothness parameters.}
The creation of the right-hand sides is the same as in the previous paragraph. Thus, we will only describe the creation of the matrices $A$, which define the strong convexity parameter $\mu$ and smoothness parameter $L$. First of all, we restrict to a diagonal matrix. Further, since we found the strong convexity parameter $\mu$ to have only a negligible influence in previous experiments (if the problem is not generally well-conditioned, in which case one would not have to learn anything), we fix it (typically $\mu = 0.05$) and only vary the smoothness parameter $L$. First, we sample $N_{prior} + N_{train} + N_{test}$ smoothness parameters uniformly at random in $[1, 5000]$. Then, for each smoothness parameter we create the matrix $A$ by linearly interpolating between $\sqrt{\mu}$ and $\sqrt{L}$ and inserting these elements (randomly permutated) into the diagonal of $A$. 

\subsection{Learning Procedure}
At first, we setup the sufficient statistics $T$ and the natural parameters $\eta$ as functions in $\rvec{\alpha}$ and $\lambda$, which can be called during training. We hand these, together with the specified priors, over to the general implementation of the learning procedure, which performs the following steps: 
\begin{itemize}
    \item[i)] First, we create samples from the initial prior (depending on the experiment between 50 and 500).
    \item[ii)] Then we evaluate the sufficient statistics T on these samples and find $\argmin_{\lambda \in \Lambda} F(\lambda)$ by a simple grid search. \textit{For this we use a linear grid $\Lambda$ over $(0,1]$ with 25000 entries (note that this corresponds to $\log(\vert \Lambda \vert) \approx 10$ and has, compared to solving the minimization problems during learning, a negligible computational cost).} Note that this also directly yields the PAC-bound.
    \item[iii)] Then, we calculate the posterior density on these samples through the formula for the Gibbs posterior, i.e., if $f$ denotes the density of the prior (w.r.t. the Lebesgue measure), we calculate $f(\alpha_i) \frac{\exp(\sprod{\eta(\lambda)}{T(\alpha_i)})}{\mathbb{E}_{\rvec{H}}[\exp(\sprod{\eta(\lambda)}{T})]}$ for every sample $\alpha_i$. Here we use the empirical mean as approximation for the integral. 
    \item[iv)] Finally, we normalize the resulting values, such that we have a distribution over these samples.
\end{itemize}
Through this, we effectively build a discrete distribution. For visualization purposes, we take a single instance (the argmax from the discrete posterior) as learned hyperparameter.

\subsection{Construction of the Prior}\label{Sec_ConstructionPrior}
If we actively construct the prior for a given hyperparameter (instead of using a fixed one as in the first experiment), we do this in an iterative fashion (typically two iterations) on a separate dataset: 
\begin{itemize}
    \item[i)] First, since we assume that we have access to the standard choice of the hyperparameters $\alpha_{std}$, we put a uniform prior around $\alpha_{std}$, i.e., $\mathcal{U}[C_1, C_2]$, where $C_1 < \alpha_{std}$ and $C_2 > \alpha_{std}$ depend on the user-specified convergence probability, i.e., they are chosen more "aggressive", if a smaller convergence guarantee has to be satisfied. \textit{In our experiments, we actually used $C_1 = \frac{0.5}{\epsilon_{conv}} \frac{2}{L_{max}}$ and $C_2 = \frac{3}{\epsilon_{conv}} \frac{2}{L_{max}}$ for the step-size parameter and $C_1 = \frac{1}{2} \beta_{std}$ and $C_2 = 2 \beta_{std}$ for the extrapolation parameter. Here $\epsilon_{conv}$ denotes the user-specified convergence probability. Initially, we started also with more "aggressive" values for the extrapolation parameter depending on the convergence probability. However, we found that the learned values almost exclusively ended up in that range, such that we directly restricted it.}
    \item[ii)] Then we run the learning procedure with this prior dataset. As described above, this yields a discrete distribution over some samples from the initial prior. From these samples, we retain only those that satisfy the user-specified convergence probability (see Section \ref{Sec_EnsuringConvProb}) and, if these are many, only those with highest posterior density.
    \item[iii)] Then we build a new uniform distribution $\mathcal{U}[a, b]$ as initial distribution for the next iteration (i.e., start from ii) again). For this, we use the standard estimators for $a$ and $b$, i.e., $\min$ and $\max$ over the remaining samples. 
\end{itemize}  
Note that this procedure is contractive, i.e., it does not yield a distribution that puts mass outside the very first initial distribution.

\subsection{Ensuring a Certain Convergence Probability} \label{Sec_EnsuringConvProb}
As described above and in the main text, ensuring the convergence probability is part of the construction of the prior. For this, we simply split the prior data set into two parts of size $N_{prior, 1}$ and $N_{prior, 2}$ (typically $N_{prior, 1} \approx N_{prior, 2}$). The first one is used in the learning procedure in the construction of the prior as described above in Section \ref{Sec_ConstructionPrior}, and the second one is used as a separate test set to check for the convergence probability. Here we use the standard estimator for the binomial distribution $\hat{p}_{conv}(\alpha) = \frac{N_{conv}}{N_{prior, 2}}$. Based on this estimate, we only keep those samples in Section \ref{Sec_ConstructionPrior} that satisfy the user-specified convergence probability during the construction of the prior. Hence, since the construction of the prior is contractive (as described in Section \ref{Sec_ConstructionPrior}), this constrains the prior to only put mass on regions that satisfy the convergence guarantee. However, as seen in the experiments, it is also partly rather conservative.

\newpage
\section{Further Remarks \& Definitions}
This section provides further details on statements made throughout the paper for which no proof has yet been provided. Furthermore, we provide a few (standard) definitions that were used in the main text, but might not be known by every reader.

\subsection{Further Definitions}\label{Appendix_Definitions}
The following two definitions are needed in Definition \ref{Def_Data_Dependent_Distribution} of a data-dependent distribution.
\begin{Def}[Polish Space]
    A topological space $\mathcal{X}$ that is separable, i.e. it has a countable dense subset, and completely metrizable, i.e. there is a complete metric that generates the topology, is called a Polish space.
\end{Def}

\begin{Def}[Markov Kernel]
    Let $(\Omega_1, \mathcal{A}_1)$, $(\Omega_2, \mathcal{A}_2)$ be measurable spaces. A map
    $$
        \kappa: \Omega_1 \times \mathcal{A}_2 \to [0,1] 
    $$
    is called a Markov kernel, if:
    \begin{itemize}
        \item[i)] For every $\omega_1 \in \Omega_1$, the map $A_2 \mapsto \kappa(\omega_1, A_2)$ is a probability measure on $\mathcal{A}_2$.
        \item[ii)] For every $A_2 \in \mathcal{A}_2$, the map $\omega_1 \to \kappa(\omega_1, A_2)$ is measurable.
    \end{itemize}
\end{Def}

\subsection{On the Measurability Assumption of $p(\alpha)$} 

\begin{Lemma}\label{Proof_Rem_Meas_p}
    Let $\pspace$ be a probability space and $\rvec{H}: \pspace \to \mathcal{H}$, $\rvec{S}: \pspace \to \Theta$ be random variables. Assume that $\Theta$ is a Polish space and that $\rvec{S}$ and $\rvec{H}$ are independent. Then there is a measurable function $p': \mathcal{H} \to [0,1]$, such that it holds:
    $$
        p' \circ \rvec{H} = p \circ \rvec{H} 
        \qquad \mathbb{P}_{\rvec{H}}-a.s.
    $$
\end{Lemma}
\begin{proof}
    Denote by:
    $$
        g(\alpha, \theta) := \ell(\mathcal{A}(\alpha, \theta), \theta) - \ell(x^{(0)}, \theta) \,.
    $$
    Then, by definition, it holds for $\alpha \in \mathcal{H}$:
    $$
        p(\alpha) = \mathbb{P}_\rvec{S}\{g_\alpha(\rvec{S}) \le 0\} = \int_\Theta \mathds{1}_{(-\infty, 0]}(g_\alpha(\theta)) \ \mathbb{P}_\rvec{S}(d\alpha)\,,
    $$
    where $g_\alpha(\theta): \Theta \to \mathbb{R}, \ \theta \mapsto g(\alpha, \theta)$ denotes $g$ with fixed argument $\alpha \in \mathcal{H}$. Since $\Theta$ is a Polish space, there exists a regular version of the conditional probability
    $$
        (\alpha, B) \mapsto \mathbb{P}_{\rvec{S} \vert \rvec{H} = \alpha}(B)
    $$
    of $\rvec{S}$ given $\rvec{H} = \alpha$. With this, for every measurable function $f: \mathcal{H} \times \Theta \to \mathbb{R}$, such that $\mathbb{E}[f(\rvec{H}, \rvec{S})]$ exists, it holds (see e.g. \cite[p.124, Thm. 1.122]{W2013}):
    $$
    \mathbb{E}[f(\rvec{H}, \rvec{S}) \ \vert \ \rvec{H}=\alpha] = \int_\Theta f_\alpha(\theta) \ \mathbb{P}_{\rvec{S} \vert \rvec{H}=\alpha}(d\theta) 
    \qquad \mathbb{P}_\rvec{H}-a.s.
    $$
    By independence, one further gets that (see e.g. \cite[p.126, Thm. 1.123]{W2013}):
    $$
        \mathbb{P}_{\rvec{S} \vert \rvec{H}=\alpha} = \mathbb{P}_\rvec{S} 
        \qquad \mathbb{P}_\rvec{H}-a.s.
    $$
    Hence, in total one gets $\mathbb{P}_\rvec{H}$-a.s.:
    \begin{align*}
        \mathbb{P}\{g(\rvec{H}, \rvec{S}) \ \vert \ \rvec{H}=\alpha\}
        &= \int_\Theta \mathds{1}_{(-\infty, 0]}(g_\alpha(\theta)) \ \mathbb{P}_{\rvec{S} \vert \rvec{H}=\alpha}(d\theta) \\
        &= \int_\Theta \mathds{1}_{(-\infty, 0]}(g_\alpha(\theta)) \ \mathbb{P}_{\rvec{S}}(d\theta) \\
        &= p(\alpha) \,.
    \end{align*}
    Since $\alpha \mapsto \mathbb{P}\{g(\rvec{H}, \rvec{S}) \ \vert \ \rvec{H}=\alpha\}$ is measurable by definition of regular conditional probabilities, the claim follows.
\end{proof}

\subsection{On the Finiteness Assumption of $\Lambda$} 

We denote the open ball of radius $r$ around a point $x$ by $\mathcal{B}(x; r)$ and the corresponding closed ball by $\mathcal{B}[x; r]$. For a set $S$, the notation $\vert S \vert$ denotes the cardinality of $S$.

\begin{Def}[Totally Bounded Space]
    A metric space $(\mathcal{X}, d)$ is called \emph{totally bounded}, if for every $\epsilon > 0$ there exists $n \in \mathbb{N}$, $x_1,...,x_n \in \space{X}$, such that
    $$
        \mathcal{X} \subseteq \bigcup_{i=1}^n \mathcal{B}(x_i;\epsilon) \,.
    $$
\end{Def}

The typical example of a totally bounded space is a compact space. The important property of this space, which is used in the following, is that they have a finite covering number. In the end, this allows again to apply the union bound argument.

\begin{Def}[$\delta$-Covering Number]
    Let $(\mathcal{X}, d)$ be a totally bounded metric space and let $\delta > 0$. A \emph{proper $\delta$-covering} of $\mathcal{X}$ is a finite set $X_\delta \subset \mathcal{X}$, such that
    $$
        \mathcal{X} \subseteq \bigcup_{x \in X_\delta} \mathcal{B}[x;\delta] \,.
    $$
    The minimal cardinality of any $\delta$-covering is denoted $\coveringNumber{\delta}{\mathcal{X}}$ and is called the \emph{$\delta$-covering number} of $\mathcal{X}$:
    $$
        \coveringNumber{\delta}{\mathcal{X}} := \min \{\vert X_\delta \vert \ : \ \text{$X_\delta$ is a proper $\delta$-covering of $\mathcal{X}$}\} \,.
    $$
\end{Def}

Taken together, and using the proof of Theorem \ref{PAPER_PACBayesExponentialFamiliy} as entry point, one gets the following Lemma. This is a direct generalization of the result in Theorem \ref{PAPER_PACBayesExponentialFamiliy}, as in the case where $\Lambda$ is finite, $\Lambda$ can be covered by itself, such that it holds $\coveringNumber{\delta}{\Lambda} = \vert \Lambda \vert$ and $C = 0$.

\begin{Lemma}\label{proof_LambdaFiniteNotNecessary}
    Let $(\Lambda, d)$ be a totally bounded metric space and let $\delta > 0$. Assume that there is a constant $C := C(\delta)$, such that for all $\Tilde{\lambda} \in \Lambda$ it holds: 
    $$
    \sup_{\lambda \in \closedBall{\Tilde{\lambda}}{\delta}} \kappa(\lambda) - \kappa(\Tilde{\lambda}) \le C \,.
    $$
    Finally, assume that $\mathbb{P}_{\rvec{D}_N} \{ \kappa(\lambda) > s  \} \le \exp(-s)$ for all $s \in \mathbb{R}$, $\lambda \in \Lambda$. Then it holds that:
    $$
        \mathbb{P}_{\rvec{D}_N} \Bigl\{ \sup_{\lambda \in \Lambda} \kappa(\lambda) \le \log \Bigl( \frac{\coveringNumber{\delta}{\Lambda}}{\epsilon} \Bigr) + C \Bigr\} \ge 1-\epsilon \,.
    $$
\end{Lemma}
\begin{proof}
    Since $(\Lambda, d)$ is a totally bounded metric space, its covering number $\coveringNumber{\delta}{\Lambda}$ is well-defined and finite. For notational simplicity, set $N := \coveringNumber{\delta}{\Lambda}$. Hence, there are $\lambda_1, ..., \lambda_N \in \Lambda$, such that:
    $$
        \Lambda \subseteq \bigcup_{i=1}^N \closedBall{\lambda_i}{\delta} \,.
    $$
    Therefore, one directly gets:
    $$
        \sup_{\lambda \in \Lambda} \kappa(\lambda) \le \max_{i=1,...,N} \ \sup_{\lambda \in \closedBall{\lambda_i}{\delta}} \kappa(\lambda) \,.
    $$
    Further, by assumption it holds:
    $$
        \sup_{\lambda \in \closedBall{\lambda_i}{\delta}} \kappa(\lambda) 
        = \kappa(\lambda_i) + \sup_{\lambda \in \closedBall{\lambda_i}{\delta}} \Bigr(\kappa(\lambda) - \kappa(\lambda_i) \Bigl) 
        \le \kappa(\lambda_i) + C \,.
    $$
    Hence, in total one gets for $s \in \mathbb{R}$:
    \begin{align*}
        \mathbb{P}_{\rvec{D}_N} \Bigl\{ \sup_{\lambda \in \Lambda} \kappa(\lambda) > s \Bigr\} 
        &\le \mathbb{P}_{\rvec{D}_N} \Bigl\{ \max_{i=1,...,N} \kappa(\lambda_i) + C > s \Bigr\}  \\
        &= \mathbb{P}_{\rvec{D}_N} \Bigl\{ \bigcup_{i=1}^N \{\kappa(\lambda_i) + C > s\} \Bigr\} \\
        &\le \sum_{i=1}^N \mathbb{P}_{\rvec{D}_N} \Bigl\{ \kappa(\lambda_i) + C > s \Bigr\} \\
        &\le \sum_{i=1}^N \exp(C-s) \\
        &= N \exp(C-s) \,.
    \end{align*}
    Since $\epsilon = N \exp(C-s) \ \iff s = \log\Bigl( \frac{N}{\epsilon} \Bigr) + C$, one gets:
    $$
    \mathbb{P}_{\rvec{D}_N} \Bigl\{ \sup_{\lambda \in \Lambda} \kappa(\lambda) > \log\Bigl( \frac{N}{\epsilon} \Bigr) + C \Bigr\} \le \epsilon\,.
    $$
    Taking the complementary event yields the result:
    $$
    \mathbb{P}_{\rvec{D}_N} \Bigl\{ \sup_{\lambda \in \Lambda} \kappa(\lambda) \le \log\Bigl( \frac{N}{\epsilon} \Bigr) + C \Bigr\} \ge 1-\epsilon\,.
    $$
\end{proof}

\newpage
\bibliography{z_bib.bib}

\end{document}